\documentclass[graybox]{svmult}

\usepackage{helvet}         
\usepackage{courier}        
\usepackage{type1cm}        
\usepackage{makeidx}         
\usepackage{graphicx}        
\usepackage{multicol}        
\usepackage[bottom]{footmisc}
\usepackage{url}
\usepackage{latexsym}
\usepackage{amssymb}
\usepackage[nonamelimits]{amsmath}

\usepackage[pdftex,
               pdftitle={Exact fit of simple finite mixture models},
               pdfsubject={},
               pdfkeywords={},
               pdfauthor={Dirk Tasche},
               pdfstartview=FitR,
               breaklinks=true,
               colorlinks=true]{hyperref}
							
\newtheorem{assumption}{Assumption}

\makeindex             

\begin{document}

\title*{Exact fit of simple finite mixture models}
\author{Dirk Tasche}
\institute{Dirk Tasche \at Prudential Regulation Authority, \email{dirk.tasche@gmx.net}\newline
The author currently works at the Prudential Regulation Authority (a
division of the Bank of England). He is also a visiting professor at
Imperial College, London. 
The opinions expressed in this note are those of the author 
and do not necessarily reflect views of the Bank of England.}
%
%
\maketitle


\abstract{How to forecast next year's portfolio-wide credit default rate based on last year's default observations and the current score distribution? A classical approach to this problem consists of fitting a mixture of the conditional score distributions observed last year to the current score distribution. This is a special (simple) case of a finite mixture model where the mixture components are fixed and only the weights of the components are estimated. The optimum weights provide a forecast of next year's portfolio-wide default rate. We point out that the maximum-likelihood (ML) approach to fitting the mixture distribution not only gives an optimum but even an exact fit if we allow the mixture components to vary but keep their density ratio fix. From this observation we can conclude that the standard default rate forecast based on last year's conditional default rates will always be located between last year's portfolio-wide default rate and the ML forecast for next year. 
As an application example, then cost quantification is discussed. We also discuss how the mixture model based estimation methods can be used to forecast total loss. This involves the reinterpretation of an individual classification problem as a collective quantification problem.}

\section{Introduction}
\label{se:intro}

The study of finite mixture models was initiated in the 1890s by Karl Pearson when he wanted to model multimodal densities.
Research on finite mixture models continued ever since but its focus changed over time as further areas of application were identified
and available computational power increased. More recently the natural connection between finite mixture models and classification
methods with their applications in fields like machine learning or credit scoring began to be investigated in more detail.
In these applications, often it can be assumed that the mixture models are \emph{simple} in the sense
that the component densities are known (ie there is no dependence on unknown parameters) but their weights are unknown.

In this note, we explore a specific property of simple finite mixture models, namely that their maximum likelihood (ML) estimates provide 
an exact fit if they exist, and some consequences of this property. 
In doing so, we extend the discussion of the case 'no independent estimate of unconditional
default probability given' from \cite{tasche2013art} to the multi-class case and general probability spaces.

In Section~\ref{se:ExactFit}, we present the result on the exact fit property in a general simple finite mixture model context.
In Section~\ref{se:scaled}, we discuss the consequences of this result for classification and quantification problems and compare
the ML estimator with other estimators that were proposed in the literature. In Section~\ref{se:cost}, we revisit the
cost quantification problem as introduced in \cite{forman2008quantifying} as an application.
In Section~\ref{se:loss}, we illustrate by a stylised example
from mortgage risk management how the estimators discussed before can be deployed for the forecast of expected loss rates.
Section~\ref{se:conclusions} concludes the note.


\section{The exact fit property}
\label{se:ExactFit}

We discuss the properties of the ML estimator of the weights in a simple finite mixture model in a 
general setting which may formally be
described as follows:
\begin{assumption}\label{as:basic} 
$\mu$ is a measure on $(\Omega, \mathcal{H})$. $g>0$ is a probability density with respect to $\mu$.
Write $\mathrm{P}$ for the probability measure with $\frac{d \mathrm{P}}{d \mu} = g$. Write $\mathrm{E}$ for the
expectation with regard to $\mathrm{P}$.
\end{assumption}
In applications, the measure $\mu$ will often be a multi-dimensional Lebesgue measure or a counting measure. We study
this problem:

\noindent\textbf{Problem.} \emph{Approximate $g$ by a mixture of probability $\mu$-densities $f_1, \ldots, f_k$, ie 
$g \approx \sum_{i=1}^k p_i\,f_i$ for suitable $p_1, \ldots, p_k \ge 0$ with $\sum_{i=1}^k p_i = 1$.}

In the literature, most of the time a sample version of the problem (ie with $\mu$ as an empirical measure) is discussed.
Often the component densities $f_i$ depend on parameters that have to be estimated in addition to the weights $p_i$ (see \cite{redner1984mixture},
or \cite{fruhwirth2006finite} and for more recent surveys \cite{schlattmann2009medical}). In this note, we consider 
the simple case where the component densities $f_i$ are assumed to be known and fixed. This is a standard assumption
for classification (see \cite{MorenoTorres2012521}) and quantification (see \cite{forman2008quantifying}) problems. 

Common approaches to the approximation problem are
\begin{itemize}
	\item \textbf{Least Squares.} Determine $\min_{(p_1, \ldots, p_k)} \int\bigl(g - \sum_{i=1}^k p_i\,f_i\bigr)^2\,d \mu$ or its
	weight\-ed version $\min_{(p_1, \ldots, p_k)} \int g\,\bigl(g - \sum_{i=1}^k p_i\,f_i\bigr)^2\,d \mu$ (see \cite{hofer2013drift}
	and \cite{hopkins2010method} for recent applications to credit risk and text categorisation respectively). 
	The main advantage of the least squares approach compared with other approaches comes from the fact that
	closed-form solutions are available.
	\item \textbf{Kullback-Leibler (KL) distance.} Determine 
	$$\min_{(p_1, \ldots, p_k)} \int g\,\log\left(\frac{g}{\sum_{i=1}^k p_i\,f_i}\right) d\mu$$
	(see \cite{duPlessis2014110} for a recent discussion).	
\end{itemize}
In the following we examine a special property of the KL approximation which we call the \emph{exact fit}
property. First we note two alternative representations of the KL distance (assuming all integrals are well-defined and finite):
\begin{align}
\int g\,\log\left(\frac{g}{\sum_{i=1}^k p_i\,f_i}\right) d\mu & = 
		\int g\,\log(g)\,d\mu -\int g\,\log\left(\sum_{i=1}^k p_i\,f_i\right) d\mu \label{eq:KL1}\\
		& = \int g\,\log(g)\,d\mu -\int g\,\log(f_k)\,d\mu \notag \\
		& \qquad - \int g\,\log\left(1 +\sum_{i=1}^{k-1} p_i\,(X_i-1)\right) d\mu, \label{eq:KL2}
	\end{align}
with $X_i = \frac{f_i}{f_k}$, $i = 1, \ldots, k-1$ if $f_k > 0$.

The problem in \eqref{eq:KL1} to maximise $\int g\,\log\left(\sum_{i=1}^k p_i\,f_i\right) d\mu$  was
studied in \cite{saerens2002adjusting} (with $g = \text{empirical measure}$
which gives ML estimators). The authors suggested the Expectation-Maximisation (EM) algorithm 
involving conditional class probabilities for determining
the maximum. This works well in general but sometimes suffers from very slow convergence.

The ML version of \eqref{eq:KL1} had before been studied in \cite{peters1976numerical}.  There the authors analysed 
the same iteration procedure which they stated, however, in terms of densities instead of conditional probabilities. 
In \cite{peters1976numerical}, the iteration was derived differently to \cite{saerens2002adjusting}, 
namely by studying the gradient of the likelihood function. 
We revisit the approach from \cite{peters1976numerical}
from a different angle by starting from \eqref{eq:KL2}.

There is, however, the complication that $g\,\log\left(1 +\sum_{i=1}^{k-1} p_i\,(X_i-1)\right)$ is not necessarily 
integrable. But this observation does not apply to the gradient with respect to $(p_1, \ldots, p_{k-1})$. We therefore focus
on investigating the gradient. With $X_i$ as in \eqref{eq:KL2}, let
\begin{gather}
F(p_1, \ldots, p_{k-1}) \stackrel{\text{def}}{=} \int g\,\log\left(1 +\sum_{i=1}^{k-1} p_i\,(X_i-1)\right) \,d\mu,\label{eq:max}\\
(p_1, \ldots, p_{k-1}) \in S_{k-1} \stackrel{\text{def}}{=} \bigl\{(s_1, \ldots, s_{k-1}): 
s_1 > 0, \ldots, s_{k-1}>0, \sum_{i=1}^{k-1} s_i < 1\bigr\}.\notag
\end{gather}
From this we obtain for the gradient of $F$
\begin{align}
	G_j(p_1, \ldots, p_{k-1}) & \stackrel{\text{def}}{=} \frac{\partial F}{\partial p_j}(p_1, \ldots, p_{k-1})\notag\\ & =
	\int \frac{g\,(X_j-1)}{1 +\sum_{i=1}^{k-1} p_i\,(X_i-1)}\,d \mu, \quad j = 1, \ldots, k-1.\label{eq:grad}
\end{align}
$G_j(p_1, \ldots, p_{k-1})$ is well-defined and finite for $(p_1, \ldots, p_{k-1}) \in S_{k-1}$:
\begin{equation}\label{eq:finite}
\begin{split}
|G_j(p_1, \ldots, p_{k-1})| & \le \int \frac{g\,(X_j+1)}{1 - \sum_{i=1}^{k-1} p_i +\sum_{i=1}^{k-1} p_i\,X_i}\,d \mu \\
& \le \frac1{p_j} + \frac1{1 - \sum_{i=1}^{k-1} p_i} < \infty.
\end{split}
	\end{equation}
We are now in a position to state the main result of this note.
\begin{theorem}\label{th:ratio}
Let $g$ and $\mu$ be as in Assumption~\ref{as:basic}. Assume that $X_1$, $\ldots$, $X_{k-1} \ge 0$ 
as defined for \eqref{eq:KL2} are $\mathcal{H}$-measurable
functions on $\Omega$. Suppose there is a vector $(p_1, \ldots, p_{k-1}) \in S_{k-1}$ such that 
\begin{equation}\label{eq:zero}
0 \ = \ G_i(p_1, \ldots, p_{k-1}), \qquad i = 1, \ldots, k-1,	
\end{equation}
for $G_i$ as defined in \eqref{eq:grad}.
Define $p_k = 1-\sum_{j=1}^{k-1} p_j$. Then the following two statements hold:
\begin{itemize}
	\item[a)] $g_i = \frac{g\,X_i}{1 +\sum_{j=1}^{k-1} p_j\,(X_j-1)}$, $i = 1, \ldots, k-1$, and
	$g_k = \frac{g}{1 +\sum_{j=1}^{k-1} p_j\,(X_j-1)}$ are probability densities with respect to $\mu$ such
	that $g = \sum_{i=1}^k p_i\,g_i$ and $X_i =\frac{g_i}{g_k}$, $i = 1, \ldots, k-1$.
	\item[b)] Let $X_1 -1$, $\ldots$, $X_{k-1}-1$ additionally be linearly independent, ie 
	$$\mathrm{P}\left[\sum_{i=1}^{k-1} a_i\,(X_i-1) = 0\right] = 1\quad \text{implies}\quad 0 = a_1 = \ldots = a_{k-1}.$$
	Assume that $h_1, \ldots, h_{k-1} \ge 0$, $h_k > 0$ are $\mu$-densities of probability measures on $(\Omega, \mathcal{H})$ 
	and $(q_1, \ldots, q_{k-1}) \in S_{k-1}$ with $q_k = 1- \sum_{i=1}^{k-1} q_i$ is such that 
	$g = \sum_{i=1}^k q_i\,h_i$ and $X_i =\frac{h_i}{h_k}$, $i = 1, \ldots, k-1$. Then it follows that
	$q_i = p_i$ and $h_i = g_i$ for $g_i$ as defined in a), for $i = 1, \ldots, k$.
\end{itemize}
\end{theorem}

\begin{proof} For a) we have to show that $\int g_i\,d \mu = 1$, $i = 1, \ldots, k$. The other claims 
are obvious. By \eqref{eq:grad}, Eq.~\eqref{eq:zero} implies that $\int g_i\,d \mu = \int g_k\,d \mu$, $i = 1, \ldots, k-1$.
This in turn implies
$$1 = \int g\,d \mu = \sum_{i=1}^k p_i\,\int g_i\,d \mu = \int g_k\,d \mu\,\sum_{i=1}^k p_i = \int g_k\,d \mu.$$
Hence $\int g_i\,d \mu = 1$ for all $i$.

With regard to b) observe that 
\begin{equation*}
\frac{\partial G_i}{\partial q_j}(q_1, \ldots, q_{k-1}) \ 
=\ -\int \frac{g\,(X_i-1)\,(X_j-1)}
{\big(1 +\sum_{\ell=1}^{k-1} q_\ell\,(X_\ell-1)\bigr)^2}\,d \mu.
\end{equation*}
Like in \eqref{eq:finite}, for $(q_1, \ldots, q_{k-1}) \in S_{k-1}$ it can easily be shown that $\frac{\partial G_i}{\partial q_j}$ is well-defined and finite.
Denote by $J = \bigl(\frac{\partial G_i}{\partial q_j}\bigr)_{i, j= 1, \ldots, k-1}$ the Jacobi matrix
of the $G_i$ and $\mathbf{X} = (X_1-1, \ldots, X_{k-1}-1)$. Let $a \in \mathbb{R}^{k-1}$ and $a^T$ be the
transpose of $a$. Then it holds that
$$a\,J\,a^T \ = \ - \int g\,\frac{(\mathbf{X}\,a^T)^2}{\big(1 +\sum_{\ell=1}^{k-1} q_\ell\,(X_\ell-1)\bigr)^2}\,d \mu \le 0.$$
In addition, by assumption on the linear independence of the $X_i-1$, $0 = a\,J\,a^T$ implies $a = 0 \in \mathbb{R}^{k-1}$.
Hence $J$ is negative definite. From this it follows by the mean value theorem that the solution $(p_1, \ldots, p_{k-1})$ of
\eqref{eq:zero} is unique in $S_{k-1}$. 

By assumption on $(q_1, \ldots, q_k)$ and $h_1, \ldots, h_k$ we have for $i = 1, \ldots, k-1$
\begin{equation}\label{eq:solution}
\begin{split}
1 & = \int h_i\,d \mu = \int \frac{g\,h_i}{\sum_{j=1}^k q_j\,h_j}\,d \mu = 
\int \frac{g\,X_i}{1+\sum_{j=1}^{k-1} q_j\,(X_j-1)}\,d \mu,  \quad\text{and}\\
1 & = \int h_k\,d \mu = \int \frac{g\,h_k}{\sum_{j=1}^k q_j\,h_j}\,d \mu = 
\int \frac{g}{1+\sum_{j=1}^{k-1} q_j\,(X_j-1)}\,d \mu.
\end{split}	
\end{equation}
Hence we obtain $0 = \int \frac{g\,(X_i-1)}{1+\sum_{j=1}^{k-1} q_j\,(X_j-1)}\,d \mu$, $i = 1, \ldots, k-1$.
By uniqueness of the solution of \eqref{eq:zero} it follows that $q_i = p_i$, $i = 1, \ldots, k$.
With this, it can be shown similarly to \eqref{eq:solution} that
$$\int_H h_i\,d \mu \ = \ \int_H g_i\,d \mu, \quad H \in \mathcal{H}, \ i = 1, \ldots, k.$$
This implies $h_i = g_i$, $i = 1, \ldots, k$. \hfill\qed
\end{proof}

\begin{remark}\label{rm:KL}
If the KL distance on the lhs of \eqref{eq:KL1} is well-defined and finite for all $(p_1, \ldots, p_{k-1}) \in S_{k-1}$
then under the assumptions of Theorem~\ref{th:ratio} b) there is a unique $(p^\ast_1, \ldots, p^\ast_{k-1}) \in S_{k-1}$
such that the KL distance of $g$ and $\sum_{i=1}^k p_i^\ast\,f_i$ is minimal. In addition, by Theorem~\ref{th:ratio} a), there
are densities $g_i$ such that the KL distance of $g$ and $\sum_{i=1}^k p_i^\ast\,g_i$ is zero -- this is the 
exact fit property of simple finite mixture models alluded to in the title of this note. \hfill \qed
\end{remark}

\begin{remark}\label{rm:difficult}
\begin{itemize}
	\item[a)] Theorem~\ref{th:ratio} b) provides a simple condition for uniqueness of the solution to \eqref{eq:zero}.
In the case $k=2$ this condition simplifies to
\begin{subequations}
\begin{equation}
	\mathrm{P}[X_1 = 1]\ <\ 1.
\end{equation}
For $k > 2$ there is no similarly simple condition for the existence of a solution in $S_{k-1}$. 
However, as noted in \cite{titterington1985statistical} (Example 4.3.1), there is a simple necessary and
sufficient condition for the existence of a solution in $S_1 = (0,1)$ to \eqref{eq:zero} in the case $k=2$:
\begin{equation}\label{eq:iff}
	\mathrm{E}[X_1]\ >\ 1 \qquad\text{and}\qquad \mathrm{E}\bigl[X_1^{-1}\bigr]\ >\ 1.
\end{equation} 
\end{subequations}
\item[b)] Suppose we are in the setting of Theorem~\ref{th:ratio} a) with $k > 2$ and all $X_i >0$. Hence there are 
$\mu$-densities $g_1, \ldots, g_k$, $(p_1, \ldots, p_{k-1})\in S_{k-1}$, and $p_k = 1- \sum_{i=1}^{k-1} p_i$ such that
$g = \sum_{i=1}^k p_i\,g_i$ and $X_i = \frac{g_i}{g_k}$, $i = 1, \ldots, k$. Let $\overline{g} = \frac{\sum_{j=2}^k p_j\,g_j}{1-p_1}$.
Then we have another decomposition of $g$, namely $g = p_1\,g_1 + (1-p_1)\,\overline{g}$, with $p_1 \in (0,1)$. The proof of
Theorem~\ref{th:ratio} b) shows that, with $\overline{X} = \frac{g_1}{\overline{g}}$ this implies
\begin{equation}\label{eq:k=2}
	0 \ = \ \int \frac{g\,(\overline{X}-1)}{1+p_1\,(\overline{X}-1)}\,d \mu.
\end{equation}
By a) there is a solution $p_1 \in (0,1)$ to \eqref{eq:k=2} if and only if 
$\mathrm{E}[\overline{X}] > 1$ and $\mathrm{E}[\bigl(\overline{X})^{-1}\bigr] > 1$.
\item[c)] As mentioned in a), an interesting question for the application of Theorem~\ref{th:ratio} is how to find out whether or not
there is a solution to \eqref{eq:zero} in $S_{k-1}$ for $k>2$. The iteration suggested in \cite{peters1976numerical} and
\cite{saerens2002adjusting} will correctly converge to a point on the boundary of $S_{k-1}$ if there is no solution in the
interior (Theorem~2 of \cite{peters1976numerical}). 
But convergence may be so slow that it may remain unclear whether a component of the limit is zero (and therefore
the solution is on the boundary) or genuinely very small but positive. The straight-forward Newton-Raphson approach for determining the
maximum of $F$ defined by \eqref{eq:max} may converge faster but may also become unstable for solutions close to or on
the boundary of $S_{k-1}$.

However, in case $k > 2$ the observation made in b) suggests that the following Gauss-Seidel-type iteration works if the initial
value $(p_1^{(0)}, \ldots, p_{k-1}^{(0)})$ with $p_k^{(0)} = 1 - \sum_{i=1}^{k-1} p_i^{(0)}$ 
is sufficiently close to the solution (if any) of \eqref{eq:zero}:
\begin{itemize}
	\item Assume that for some $n \ge 0$ an approximate solution 
	$(q_1, \ldots, q_{k}) =(p_1^{(n)}$, $\ldots$, $p_{k}^{(n)})$
	has been found.
	\item For $i = 1, \ldots, k$ try successively to update $(q_1, \ldots, q_{k})$ by solving \eqref{eq:k=2} with component $i$ 
	playing the role of component $1$ in b) and $p_1 = q_i$ as well as $\overline{g} = \frac{\sum_{j=1, j \not=i}^k q_j\,g_j}{1-q_i}$.
	If for all $i=1,\ldots, k$ the sufficient and necessary condition for the updated $q_i$ to be in $(0,1)$ is not satisfied then stop --
	it then is likely that there is no solution to \eqref{eq:zero} in $S_{k-1}$. Otherwise update where possible $q_i$ with the
	solution of \eqref{eq:k=2}, resulting in $q_{i, \text{new}}$, and set 
	$$q_{j, \text{new}} = \frac{1-q_{i, \text{new}}}{1-q_i}\,q_j, \quad j \not= i.$$
	\item After step $k$ set $(p_1^{(n+1)}, \ldots, p_{k}^{(n+1)}) = (q_{1, \text{new}}, \ldots, q_{k, \text{new}})$ if
	the algorithm has not been stopped by violation of the resolvability condition for \eqref{eq:zero}. 
	\item Terminate the calculation when a suitable distance measure between successive $(p_1^{(n)}$, $\ldots$, $p_{k}^{(n)})$ is
	sufficiently small. 	\qed
	\end{itemize}
\end{itemize}
\end{remark}

\section{Application to quantification problems}
\label{se:scaled}

Finite mixture models occur naturally in machine learning contexts. Specifically, in this note 
we consider the following context:

\begin{assumption}\label{as:context}\ 
\begin{itemize}
	\item $(\Omega, \mathcal{H})$ is a measurable space. For some $k \ge 2$, $A_1, \ldots, A_k \subset \Omega$ is a partition of $\Omega$.
$\mathcal{A}$ is the $\sigma$-field generated by $\mathcal{H}$ and the $A_i$, ie
$$\mathcal{A} \ = \ \sigma(\mathcal{H} \cup \{A_1, \ldots, A_k\}) \ = \ 
\bigl\{\bigcup_{i=1}^k (A_i \cap H_i): H_1, \ldots, H_k \in \mathcal{H}\bigr\}.$$
\item $\mathrm{P}_0$ is a probability measure on $(\Omega, \mathcal{A})$ with $\mathrm{P_0}[A_i] > 0$ for $i = 1, \ldots, k$. 
$\mathrm{P}_1$ is a probability measure on $(\Omega, \mathcal{H})$. Write $\mathrm{E}_i$ for the expectation with
respect to $\mathrm{P}_i$.
\item There is a measure $\mu$ on $(\Omega, \mathcal{H})$ and $\mu$-densities $f_1, \ldots, f_{k-1} \ge 0$, $f_k > 0$ such that 
$$\mathrm{P}_0[H\,|\,A_i] \ = \ \int_H f_i\,d \mu, \qquad i=1, \ldots, k, H \in \mathcal{H}.$$
\end{itemize}
\end{assumption}
The space $(\Omega, \mathcal{A}, \mathrm{P}_0)$ describes the training set of a classifier. On the training set, for each example
both the features (expressed by $\mathcal{H}$) and the class (described by one of the $A_i$) are known. Note that
$f_k >0$ implies $A_k \notin \mathcal{H}$.

$(\Omega, \mathcal{H}, \mathrm{P}_1)$ describes the test set on which the classifier is
deployed. On the test set only the features of the examples are known. 

In mathematical terms, \emph{quantification} might be
described as the task to extend $\mathrm{P}_1$ onto $\mathcal{A}$, based on properties observed on the training set, ie of $\mathrm{P}_0$.
Basically, this means to estimate \emph{prior class probabilities} (or prevalences) $\mathrm{P}_1[A_i]$ on the test dataset.
In this note, the assumption is that $\mathrm{P}_1\bigm|\mathcal{A} \not= \mathrm{P}_0\bigm|\mathcal{A}$. In the machine learning
literature, this situation is called \emph{dataset shift} (see \cite{MorenoTorres2012521} and the references therein). 

Specifically, we consider the following two dataset shift types (according to \cite{MorenoTorres2012521}):
\begin{itemize}
	\item \textbf{Covariate shift.} $\mathrm{P}_1\bigm|\mathcal{H} \not= \mathrm{P}_0\bigm|\mathcal{H}$ but
	$\mathrm{P}_1[A_i\,|\,\mathcal{H}] = \mathrm{P}_0[A_i\,|\,\mathcal{H}]$ for $i=1, \ldots, k$. In practice, this implies
	$\mathrm{P}_1[A_i] \not= \mathrm{P}_0[A_i]$ for most if not all of the $i$.
	\item \textbf{Prior probability shift.} $\mathrm{P}_1[A_i] \not= \mathrm{P}_0[A_i]$ for at least one $i$ but
	$\mathrm{P}_1[H\,|\,A_i]$ $= \mathrm{P}_0[H\,|\,A_i]$ for $i = 1, \ldots, k$, $H\in \mathcal{H}$. This implies
	$\mathrm{P}_1\bigm|\mathcal{H} \not= \mathrm{P}_0\bigm|\mathcal{H}$ if $\mathrm{P}_0[\cdot\,|\,A_1]$, 
	$\ldots$, $\mathrm{P}_0[\cdot\,|\,A_k]$ are linearly independent.
\end{itemize}
In practice, it is likely to have $\mathrm{P}_1[A_i] \not= \mathrm{P}_0[A_i]$ for some $i$ both in case
of covariate and prior probability shift.
Hence, quantification in the sense of estimation of the $\mathrm{P}_1[A_i]$ is important both for covariate and
prior probability shifts.
	
	Under a covariate shift assumption, a natural estimator of $\mathrm{P}_1[A_i]$ is given by
	\begin{equation}\label{eq:covariate}
		\widetilde{\mathrm{P}}_1[A_i] \ = \ \mathrm{E}_1\bigl[\mathrm{P}_0[A_i\,|\,\mathcal{H}]\bigr],
		\qquad i = 1, \ldots, k.
	\end{equation}
	Under prior probability shift, the choice of suitable estimators of $\mathrm{P}_1[A_i]$ is less obvious.

The following result generalises the \emph{Scaled Probability Average} method of \cite{bella2010quantification}
to the multi-class case. It allows to derive prior probability shift estimates of prior class probabilities from
covariate shift estimates as given by \eqref{eq:covariate}.
\begin{subequations}
\begin{proposition}\label{pr:ScaledProbAv}
Under Assumption~\ref{as:context}, suppose that there are $q_1 \ge 0$, $\ldots$, $q_k \ge 0$ with
$\sum_{i=1}^k q_i =1$ such that $\mathrm{P}_1$ can be represented as a simple finite mixture as follows: 
\begin{equation}\label{eq:mix}
	\mathrm{P}_1[H] \ = \ \sum_{i=1}^k q_i\,\mathrm{P}_0[H\,|\,A_i], \quad H \in \mathcal{H}.
\end{equation}
Then it follows that 
\begin{equation}\label{eq:matrix}
	\begin{pmatrix}
		\mathrm{E}_1\bigl[\mathrm{P}_0[A_1\,|\,\mathcal{H}]\bigr]\\
		\vdots\\
		\mathrm{E}_1\bigl[\mathrm{P}_0[A_k\,|\,\mathcal{H}]\bigr]
	\end{pmatrix} \ = \ M \,\begin{pmatrix}
		q_1\\
		\vdots\\
		q_k,
	\end{pmatrix}, 
	\end{equation}
where the matrix $M = (m_{ij})_{i,j=1, \ldots,k}$ is given	by
\begin{equation}\label{eq:alter}
\begin{split}
m_{ij} & \ = \ \mathrm{E}_0\bigl[\mathrm{P}_0[A_i\,|\,\mathcal{H}]\,|\,A_j]\bigr] \\
& \ = \ \mathrm{E}_0\bigl[\mathrm{P}_0[A_i\,|\,\mathcal{H}]\,\mathrm{P}_0[A_j\,|\,\mathcal{H}]\bigr] 
\,/\, \mathrm{P}_0[A_j].
\end{split}
	\end{equation}
\end{proposition}
\end{subequations}
\begin{proof} Immediate from \eqref{eq:mix} and the definition of conditional expectation.\qed
\end{proof}
For practical purposes, the representation of $m_{ij}$ in the first row of \eqref{eq:alter} is more
useful because most of the time no exact estimate of $\mathrm{P}_0[A_j\,|\,\mathcal{H}]$ will be available. 
As a consequence there might be a non-zero difference between the values of the expectations in the
first and second row of \eqref{eq:alter} respectively. In contrast to the second row, for the derivation of the rhs of the
first row of \eqref{eq:alter}, however, no use of the specific properties of conditional expectations has
been made.

\begin{corollary}\label{co:R2}
In the setting of Proposition~\ref{pr:ScaledProbAv}, suppose that $k=2$. Define 
$$R_0^2 \ =\ \frac{\mathrm{var}_0\bigl[\mathrm{P}_0[A_1\,|\,\mathcal{H}]\bigr]}{\mathrm{P}_0[A_1]\,(1-\mathrm{P}_0[A_1])}
\in [0,1],$$
with $\mathrm{var}_0$ denoting the variance under $\mathrm{P}_0$. Then we have
\begin{equation}\label{eq:R2}
		\mathrm{E}_1\bigl[\mathrm{P}_0[A_1\,|\,\mathcal{H}]\bigr] \ = \ 
		\mathrm{P}_0[A_1]\,(1-R^2_0) + q_1\,R^2_0.
\end{equation}
\end{corollary}
\begin{proof} We start from the first element of the vector-equation \eqref{eq:matrix} and apply
some algebraic manipulations:
\begin{align*}
\mathrm{E}_1\bigl[\mathrm{P}_0[A_1\,|\,\mathcal{H}]\bigr] & = 
\tfrac{q_1}{\mathrm{P}_0[A_1]}\,\mathrm{E}_0\bigl[\mathrm{P}_0[A_1\,|\,\mathcal{H}]^2\bigr]\\ 
& \qquad\qquad +
\tfrac{1-q_1}{1-\mathrm{P}_0[A_1]}\,\mathrm{E}_0\bigl[\mathrm{P}_0[A_1\,|\,\mathcal{H}](1-\mathrm{P}_0[A_1\,|\,\mathcal{H}])\bigr]\\
& = \bigl(\mathrm{P}_0[A_1]\,(1-\mathrm{P}_0[A_1])\bigr)^{-1}\,\Big(q_1\,(1-\mathrm{P}_0[A_1])\,\mathrm{E}_0\bigl[\mathrm{P}_0[A_1\,|\,\mathcal{H}]^2\bigr]\\
& \qquad\qquad + \mathrm{P}_0[A_1]\,(1-q_1)\,\bigl(\mathrm{P}_0[A_1]-\mathrm{E}_0\bigl[\mathrm{P}_0[A_1\,|\,\mathcal{H}]^2\bigr]\bigr)\Big)\\
& = \bigl(\mathrm{P}_0[A_1]\,(1-\mathrm{P}_0[A_1])\bigr)^{-1}\,\Big(q_1\,\bigl(\mathrm{E}_0\bigl[\mathrm{P}_0[A_1\,|\,\mathcal{H}]^2\bigr]
- \mathrm{P}_0[A_1]^2\bigr) \\
& \qquad\qquad+ \mathrm{P}_0[A_1]\,\bigl(\mathrm{P}_0[A_1] - \mathrm{E}_0\bigl[\mathrm{P}_0[A_1\,|\,\mathcal{H}]^2\bigr]\bigr)\Big).
	\end{align*}
From $\mathrm{P}_0[A_1] - \mathrm{E}_0\bigl[\mathrm{P}_0[A_1\,|\,\mathcal{H}]^2\bigr] =
\mathrm{P}_0[A_1]\,(1-\mathrm{P}_0[A_1]) - \mathrm{var}_0\bigl[\mathrm{P}_0[A_1\,|\,\mathcal{H}]\bigr]$ now the
result follows. \qed
\end{proof}
By Corollary~\ref{co:R2}, for binary classification problems 
the covariate shift estimator \eqref{eq:covariate} underestimates the change in the class probabilities if the dataset shift
is not a covariate shift but a prior probability shift. See Section~2.1 of \cite{forman2008quantifying} for
a similar result for the \emph{Classify \& Count} estimator. However, according to \eqref{eq:R2} the difference between the
covariate shift estimator and the true prior probability is the smaller the greater the discriminatory power 
(as measured by the generalised $R^2$) of the classifier is. Moreover, both \eqref{eq:R2} and \eqref{eq:matrix} provide 
closed-form solutions for $q_1$, $\ldots$, $q_k$ that transform the covariate shift estimates into correct estimates
under the prior probability shift assumption. In the following the estimators defined this way are called
\emph{Scaled Probability Average} estimators.

Corollary~\ref{co:R2} on the relationship between covariate shift and Scaled Probability Average estimates
in the binary classification case
can be generalised to the relationship between covariate shift and KL distance estimates.

\begin{corollary}\label{co:R2gen}
Under Assumption~\ref{as:context}, consider the case $k=2$. Let $X_1 = \frac{f_1}{f_2}$ and suppose that
\eqref{eq:iff} holds for $\mathrm{E} = \mathrm{E}_1$ such that a solution $p_1 \in (0,1)$ of \eqref{eq:zero} exists.
Then there is some $\alpha \in [0,1]$ such that
$$\mathrm{E}_1\bigl[\mathrm{P}_0[A_1\,|\,\mathcal{H}]\bigr] \ = \ (1-\alpha)\,\mathrm{P}_0[A_1] +
\alpha\,p_1.$$
\end{corollary}	
\begin{proof} Suppose that $g>0$ is a density of $\mathrm{P}_1$ with respect to some measure $\nu$ on $(\Omega, \mathcal{H})$.
$\nu$ need not equal $\mu$ from Assumption~\ref{as:context}, and we can choose $\nu = \mathrm{P}_1$ and $g=1$ if there is no other candidate. 
By
Theorem \ref{th:ratio}~a) then there are $\nu$-densities $g_1 \ge 0$, $g_2 >0$ such that $\frac{g_1}{g_2}=X_1$ and
$g=p_1\,g_1 + (1-p_1)\,g_2$.

We define a new probability measure $\widetilde{\mathrm{P}}_0$ on $(\Omega, \mathcal{A})$ by setting
\begin{equation*}
	\begin{split}
	\widetilde{\mathrm{P}}_0[A_i] & = \mathrm{P}_0[A_i], \qquad i =1,2,\\
	\widetilde{\mathrm{P}}_0[H\,|\,A_i] & = \int_H g_i\,d \nu,\quad H\in\mathcal{H},\ i=1,2.
	\end{split}
\end{equation*}
By construction of $\widetilde{\mathrm{P}}_0$ it holds that 
$$\mathrm{P}_1[H] = p_1\,\widetilde{\mathrm{P}}_0[H\,|\,A_1] + (1-p_1)\,\widetilde{\mathrm{P}}_0[H\,|\,A_2],
\quad H\in\mathcal{H}.$$
Hence we may apply Corollary~\ref{co:R2} to obtain
$$\mathrm{E}_1\bigl[\widetilde{\mathrm{P}}_0[A_1\,|\,\mathcal{H}]\bigr] \ = \ 
		\mathrm{P}_0[A_1]\,(1-\widetilde{R}^2_0) + q_1\,\widetilde{R}^2_0,$$
where $\widetilde{R}^2_0 \in [0,1]$ is defined like $R^2_0$ with $\mathrm{P}_0$ replaced
by $\widetilde{\mathrm{P}}_0$.
Observe that also by construction of $\widetilde{\mathrm{P}}_0$ we have
$$\widetilde{\mathrm{P}}_0[A_1\,|\,\mathcal{H}] = \frac{\mathrm{P}_0[A_1]\,g_1}
{\mathrm{P}_0[A_1]\,g_1+\mathrm{P}_0[A_2]\,g_2} = 
\frac{\mathrm{P}_0[A_1]\,X_1}
{\mathrm{P}_0[A_1]\,X_1+\mathrm{P}_0[A_2]} = \mathrm{P}_0[A_1\,|\,\mathcal{H}].$$
With the choice $\alpha = \widetilde{R}^2_0$ this proves Corollary~\ref{co:R2gen}. \qed
\end{proof}

How is the KL distance estimator (or ML estimator in case of $\nu$ being the empirical measure)
of the prior class probabilities, defined by the 
solution of \eqref{eq:zero}, in general related to the 
covariate shift and Scaled Probability Average estimators?

Suppose the test  dataset differs from the training dataset by a prior probability shift with positive 
class probabilities, ie \eqref{eq:mix} applies with $q_1, \ldots, q_k > 0$. Under Assumption~\ref{as:context}
and a mild linear independence condition on the ratios of the densities $f_i$, then Theorem~\ref{th:ratio}
implies that the KL distance and Scaled Probability Average estimators give the same results. Observe that
in the context given by Assumption~\ref{as:context} the variables $X_i$ from Theorem~\ref{th:ratio} can be directly defined
as $X_i = \frac{f_i}{f_k}$, $i=1, \ldots, k-1$ or, equivalently by
\begin{equation}\label{eq:X}
	X_i \ = \ \frac{\mathrm{P}_0[A_i\,|\,\mathcal{H}]}{\mathrm{P}_0[A_k\,|\,\mathcal{H}]}\,
			\frac{\mathrm{P}_0[A_k]}{\mathrm{P}_0[A_i]}, \qquad i=1, \ldots, k-1.
\end{equation}
Representation  \eqref{eq:X} of the density ratios might be preferable in particular if the classifier involved
has been built by binary or multinomial logistic regression.

In general, by Theorem~\ref{th:ratio} the result of applying the KL distance estimator to the test feature
distribution $\mathrm{P}_1$, in the quantification
problem context described by Assumption~\ref{as:context}, is a representation of $\mathrm{P}_1$ as a mixture
of distributions whose density ratios are the same as the density ratios of the class feature distributions
$\mathrm{P}_0[\cdot\,|\,A_i]$, $i = 1, \ldots, k-1$.

Hence the KL distance estimator makes sense under an assumption of identical density ratios in the training
and test datasets. On the one hand this assumption is similar to the assumption of identical conditional class probabilities
in the covariate shift assumption but does not depend in any way on the training set prior class probabilities.
This is in contrast to the covariate shift assumption where implicitly a 'memory effect' with regard to the 
training set prior class probabilities is accepted.

On the other hand the 'identical density ratios' assumption is weaker than the 'identical densities' assumption 
(the former is implied by the latter) which is part of the prior probability assumption.

One possible description of 'identical density ratios' and the related KL distance estimator is that 
'identical density ratios' generalises 'identical densities' in such a way that exact fit of the
test set feature distribution is achieved (which by Theorem~\ref{th:ratio} is not always possible).
It therefore is fair to say that 'identical density ratios' is closer to 'identical densities' 
than to 'identical conditional class probabilities'.

Given training data with full information (indicated by the $\sigma$-field $\mathcal{A}$ in Assumption~\ref{as:context})
and test data with information only on the features but not on the classes 
($\sigma$-field $\mathcal{H}$ in Assumption~\ref{as:context}), it is not possible to decide whether the 
covariate shift or the identical density ratios assumption is more appropriate for the data. For both
assumptions result in exact fit of the test set feature distribution $\mathrm{P}_1\bigm| \mathcal{H}$ but
in general give quite different estimates of the test set prior class probabilities (see Corollary~\ref{co:R2gen} and
Section~\ref{se:loss}). 
Only if Eq.~\eqref{eq:zero}
has no solution with positive components it can be said that 'identical density ratios' does not properly
describe the test data because then there is no exact fit of the test set feature distribution. In that case
'covariate shift' might not be appropriate either but at least it delivers a mathematically consistent model of
the data.

If both 'covariate shift' and 'identical density ratios' provide consistent models (ie exact fit of 
the test set feature distribution) non-mathematical considerations of causality (are features caused by class 
or is class caused by features?) may help choosing the more suitable assumption. See \cite{fawcett2005response} 
for a detailed discussion of this issue.


\section{Cost quantification}
\label{se:cost}

'Cost quantification' is explained in \cite{forman2008quantifying} as follows: 
``The second form of the quantification task is for a common situation in
business where a cost or value attribute is associated with each case. For example,
a customer support log has a database field to record the amount of time spent to
resolve each individual issue, or the total monetary cost of parts and labor used
to fix the customer's problem. $\ldots$ The cost quantification task for machine learning: given a limited training set
with class labels, induce a cost quantifier that takes an unlabeled test set as input
and returns its best estimate of the total cost associated with each class. In other
words, return the subtotal of cost values for each class.''

Careful reading of Section~4.2 of \cite{forman2008quantifying} reveals that the favourite solutions 
for cost quantification presented by the author essentially apply only to the case where the cost attributes are
constant on the classes\footnote{%
Only then the $C^+$ as used in Equations (4) and (5) of \cite{forman2008quantifying} stand for the same 
conditional expectations. The same observation applies to $C^-$.}.

Cost quantification can be more generally treated under Assumption~\ref{as:context} of this note. 
Denote by $C$ the (random) cost associated with an example. According to the description of cost quantification quoted 
above then $C$ is actually a feature of the example and, therefore, may be considered an $\mathcal{H}$-measurable
random variable under Assumption~\ref{as:context}.

In mathematical terms, the objective of cost quantification is the estimation of the
total expected cost per class\footnote{%
For a set $M$ its indicator function $\mathbf{1}_M$ is defined as 
$\mathbf{1}_M(m) = 1$ for $m\in M$ and $\mathbf{1}_M(m) = 0$ for
$m\notin M$.}
$\mathrm{E}_1[C\,\mathbf{1}_{A_i}]$, $i = 1, \ldots, k$.

\textbf{Covariate-shift assumption.} Under this assumption we obtain
\begin{equation}\label{eq:cov.cost}
	\mathrm{E}_1[C\,\mathbf{1}_{A_i}] \ = \ \mathrm{E}_1\bigl[C\,\mathrm{P}_1[A_i\,|\,\mathcal{H}]\bigr]
	\ = \ \mathrm{E}_1\bigl[C\,\mathrm{P}_0[A_i\,|\,\mathcal{H}]\bigr].
\end{equation}
This gives a probability-weighted version of the 'Classify \& Total' estimator of \cite{forman2008quantifying}.

\textbf{'Constant density ratios' assumption.}
Let $X_i = \frac{f_i}{f_k}$, $i=1,\ldots, k-1$. If \eqref{eq:zero} (with $\mu =\mathrm{P}_1$ and $g=1$) has
a solution $p_1 >0$, $\ldots$, $p_{k-1} > 0$, $p_k = 1 - \sum_{j=1}^{k-1} p_j < 1$ then 
we can estimate the conditional class probabilities $\mathrm{P}_1[A_i\,|\,\mathcal{H}]$ by
\begin{align*}
 \mathrm{P}_1[A_i\,|\,\mathcal{H}] & = \frac{p_i\,X_i}{1+\sum_{j=1}^{k-1} p_j\,(X_j-1)},
\quad i =1,\ldots, k-1,\\
\mathrm{P}_1[A_k\,|\,\mathcal{H}] & = \frac{p_k}{1+\sum_{j=1}^{k-1} p_j\,(X_j-1)}.
\end{align*}
From this, it follows that
\begin{equation}\label{eq:ratio.cost}
\begin{split}
	\mathrm{E}_1[C\,\mathbf{1}_{A_i}]  & \ = \ p_i\,\mathrm{E}_1\left[\frac{C\,X_i}{1+\sum_{j=1}^{k-1} p_j\,(X_j-1)}\right],
	\quad i =1,\ldots, k-1,\\
	\mathrm{E}_1[C\,\mathbf{1}_{A_k}]  & \ = \ p_k\,\mathrm{E}_1\left[\frac{C}{1+\sum_{j=1}^{k-1} p_j\,(X_j-1)}\right].
\end{split}	
\end{equation}
Obviously, the accuracy of the estimates on the rhs of both \eqref{eq:cov.cost} and \eqref{eq:ratio.cost}
strongly depends on the accuracy of the estimates of $\mathrm{P}_0[A_i\,|\,\mathcal{H}]$ and the density ratios
on the training set. Accurate estimates of these quantities, in general, will make full use of the information in
the $\sigma$-field $\mathcal{H}$ (ie the information available at the time of estimation) and, because of the 
$\mathcal{H}$-measurability of $C$, of the cost feature $C$. In order to achieve this, $C$ must be used as an
explanatory variable when the relationship between the classes $A_i$ and the features as reflected in $\mathcal{H}$
is estimated (eg by a regression approach). As one-dimensional densities are relatively easy to estimate it might make sense
to deploy \eqref{eq:cov.cost} and \eqref{eq:ratio.cost} with the choice $\mathcal{H} = \sigma(C)$.

Note that this conclusion, at first glance, seems to contradict Section~5.3.1 of \cite{forman2008quantifying}. There it is recommended that
``the cost attribute almost never be
given as a predictive input feature to the classifier''. Actually, with regard to the cost quantifiers suggested in 
\cite{forman2008quantifying}, this recommendation is reasonable because the main component of the quantifiers as
stated in (6) of \cite{forman2008quantifying} is correctly specified  only if there is no dependence of the cost
attribute $C$ and the classifier. Not using $C$ as an explanatory variable, however, does not necessarily imply
that the dependence between $C$ and the classifier is weak. Indeed, if the classifier has got any predictive power and 
$C$ is on average different on the the different classes of examples then there must be a non-zero correlation between
the cost attribute $C$ and the output of the classifier.


\section{Loss rates estimation with mixture model methods}
\label{se:loss}

Theorem~\ref{th:ratio} and the results of Section~\ref{se:scaled} have obvious applications to 
the problem of forecasting portfolio-wide default rates in portfolios of rated or scored borrowers.
The forecast portfolio-wide default rate may be interpreted in an individual sense as a single borrower's
unconditional probability of default. But there is also an interpretation in a collective sense as
the forecast total proportion of defaulting borrowers.

The statements of Theorem~\ref{th:ratio} and Assumption~\ref{as:context} are agnostic in the sense of not 
suggesting an individual or collective interpretation of the models under inspection. But by explaining
Assumption~\ref{as:context} in terms of a classifier and the examples to which it is applied we 
have suggested an individual interpretation of the assumption.

However, there is no need to adopt this perspective on Assumption~\ref{as:context} and the results 
of Section~\ref{se:scaled}. Instead of interpreting $\mathrm{P}_0[A_1]$ as 
an individual example's probability of belonging to class~1 we could as well describe 
$\mathrm{P}_0[A_1]$ as the proportion of a mass or substance that has property~1. If we do so
we switch from an interpretation of probability spaces in terms of likelihoods associated 
with individuals to an interpretation in terms of proportions of parts of masses or substances.
\begin{table}[!tp]
\begin{center}
\caption{Fictitious report on mortgage portfolio exposure distribution and loss rates.}\label{tab:1}
\begin{tabular}{|l||r|r||r|c|}
\hline
\multicolumn{1}{|l||}{\bfseries LTV band}&\multicolumn{2}{c||}{\bfseries Last year}&\multicolumn{2}{c|}{\bfseries This year}\tabularnewline
\cline{2-3} \cline{4-5}
\multicolumn{1}{|l||}{}&\multicolumn{1}{c|}{\% of exposure}&\multicolumn{1}{c||}{of this \% lost}&
	\multicolumn{1}{c|}{\% of exposure}&\multicolumn{1}{c|}{of this \% lost}\tabularnewline
\hline
More than 100\% &$  10.3$&$15.0$ & $ 13.3$&?\tabularnewline\hline
Between 90\% and 100\% &$ 28.2$&$ 2.2$ & $ 24.2$&?\tabularnewline\hline
Between 70\% and 90\% &$ 12.9$&$ 1.1$ & $ 12.8$&?\tabularnewline\hline
Between 50\% and 70\% &$ 24.9$&$ 0.5$ & $ 25.4$&?\tabularnewline\hline
Less than 50\% &$ 23.8$&$ 0.2$ & $ 24.3$&?\tabularnewline\hline
\textbf{All}&$100.0$&$ 2.2$ & $100.0$&?\tabularnewline
\hline
\end{tabular}
\end{center}
\end{table}

Let us look at a retail mortgage portfolio as an illustrative example. Suppose that each mortgage
has a loan-to-value (LTV) associated with it which indicates how well the mortgage loan is secured by the
pledged property. Mortgage providers typically report their exposures and losses in tables that
provide this information per LTV-band without specifying numbers or percentages of borrowers involved. 
Table~\ref{tab:1} shows a stylised example of how such
a report might look like.

This portfolio description fits well into the framework described by Assumption~\ref{as:context}. Choose
events $H_1 =$ 'More than 100\% LTV', $H_2 =$ 'Between 90\% and 100\% LTV' and so on. Then the $\sigma$-field
$\mathcal{H}$ is generated by the finite partition $H_1, \ldots, H_5$. Similarly, choose $A_1 =$ 'lost' and
$A_2 =$ 'not lost'. The measure $\mathrm{P}_0$ describes last year's observations, $\mathrm{P}_1$ specifies
the distribution of the exposure over the LTV bands as observed at the beginning of this year -- which is
the forecast period. We can then try and replace the question marks in Table~\ref{tab:1} by deploying 
the estimators discussed in Section~\ref{se:scaled}. Table~\ref{tab:2} shows the results. 
\begin{table}[!tp]
\begin{center}
\caption{This year's loss rates forecast for portfolio from Table~\ref{tab:1}. The $R_0^2$ used for calculating
the Scaled Probability Average is 7.7\%.}\label{tab:2}
\begin{tabular}{|l||r|r|r|}
\hline
{\bfseries LTV band}&{\bfseries Covariate shift}&{\bfseries Scaled prob.~av.}&{\bfseries KL distance}\tabularnewline
&\% lost &\% lost & \% lost\tabularnewline\hline
More than 100\% & 15.0 & 34.6 &	35.4 \tabularnewline\hline
Between 90\% and 100\% &2.2 &	6.3	& 6.5\tabularnewline\hline
Between 70\% and 90\% &1.1	& 3.2	& 3.3\tabularnewline\hline
Between 50\% and 70\% &0.5	& 1.5	& 1.5 \tabularnewline\hline
Less than 50\% &0.2	& 0.6	& 0.6\tabularnewline\hline
\textbf{All}&2.8&7.3 & 7.1 \tabularnewline
\hline
\end{tabular}
\end{center}
\end{table}

Clearly, the
estimates under the prior probability shift assumptions are much more sensitive to changes of the
features (ie LTV bands) distribution than the estimate under the covariate shift assumption. Thus the theoretical 
results of Corollaries \ref{co:R2} and \ref{co:R2gen} are confirmed. But recall that there is no
right or wrong here as all the numbers in Table~\ref{tab:1} are purely fictitious. Nonetheless,
we could conclude that in applications with unclear causalities (like for credit risk measurement)
it might make sense to compute both covariate shift estimates and ML estimates (more suitable under
a prior probability shift assumption) in order to gauge the possible range of outcomes.


\section{Conclusions}
\label{se:conclusions}

We have revisited the maximum likelihood estimator (or more generally Kullback-Leibler (KL) distance estimator)
of the component weights in simple finite mixture models. We have found that (if all weights of the estimate are positive)
it enjoys an exact fit property which makes it even more attractive with regard to mathematical consistency.
We have suggested a Gauss-Seidel-type approach to the calculation of the KL distance estimator that
triggers an alarm if there is no solution with all components positive (which would indicate that the
number of modelled classes may be reduced).

In the context of two-class quantification problems, as a consequence of the exact fit property we have shown 
theoretically and by example that the straight-forward 'covariate shift' estimator of the prior class probabilities
may seriously underestimate the change of the prior probabilities if the covariate shift assumption is wrong and instead
a prior probability shift has occurred. This underestimation can be corrected by the Scaled Probability Average approach
which we have generalised to the multi-class case or the KL distance estimator. 

As an application example, we then have  discussed cost quantification, ie the attribution of total cost to
classes on the basis of characterising features when class membership is unknown. 
In addition, we have illustrated by example that
the mixture model approach to quantification is not restricted to the forecast of prior probabilities but
can also be deployed for forecasting loss rates.



\end{document}